\documentclass{article}

\usepackage{microtype}
\usepackage{graphicx}
\usepackage{subcaption}
\usepackage{booktabs}
\usepackage{authblk}
\usepackage[square,numbers]{natbib}
\usepackage{amsmath}
\usepackage{amsthm}
\usepackage{amssymb}
\usepackage{amsfonts}
\usepackage{stmaryrd}
\usepackage{bm}
\usepackage{xcolor}

\usepackage[hidelinks]{hyperref}

\newcommand{\ds}{\displaystyle}
\newcommand{\R}{\mathbb{R}}
\newcommand{\E}{\mathbb{E}}
\renewcommand{\P}{\mathbb{P}}
\newcommand{\relu}[1]{(#1)_+}
\newcommand{\pred}[1]{\llbracket #1 \rrbracket}
\newcommand{\fspace}{\mathcal{Z}}
\newcommand{\tspace}{\mathcal{T}}
\newcommand{\bids}{\mathbf{b}}
\newcommand{\asks}{\mathbf{c}}

\newcommand{\hb}{b^{(1)}}
\newcommand{\hob}{b^{(2)}}
\newcommand{\secp}{\bar{c}}
\newcommand{\mr}{{\sf MR}}
\newcommand{\sw}{{\sf SW}}

\DeclareMathOperator*{\argmax}{arg\,max}

\newtheorem{definition}{Definition}
\newtheorem{prop}[definition]{Proposition}
\newtheorem{cor}[definition]{Corollary}

\title{Learning to Clear the Market}

\author[1]{Weiran Shen\footnote{Corresponding author. 
\texttt{emersonswr@gmail.com}}}
\author[2]{S\'{e}bastien Lahaie}
\author[2]{Renato Paes Leme}

\affil[1]{Tsinghua University, Beijing, China}
\affil[2]{Google Research, New York, New York, USA}

\begin{document}

\maketitle

\begin{abstract}
  The problem of market clearing is to set a price for an item such
  that quantity demanded equals quantity supplied. In this work, we
  cast the problem of predicting clearing prices into a learning
  framework and use the resulting models to perform revenue
  optimization in auctions and markets with contextual information.
  The economic intuition behind market clearing allows us to obtain
  fine-grained control over the aggressiveness of the resulting
  pricing policy, grounded in theory. To evaluate our approach, we fit
  a model of clearing prices over a massive dataset of bids in display
  ad auctions from a major ad exchange. The learned prices outperform
  other modeling techniques in the literature in terms of revenue and
  efficiency trade-offs. Because of the convex nature of the clearing
  loss function, the convergence rate of our method is as fast as
  linear regression.
\end{abstract}


\section{Introduction}

A key difficulty in designing machine learning systems for revenue
optimization in auctions and markets is the discontinuous nature of
the problem. Consider the basic problem of setting a reserve price in
a single-item auction (e.g., for online advertising): revenue steadily
increases with price up to the point where all buyers drop out, at
which point it suddenly drops to zero. The discontinuity may average
away over a large market, but one is typically left with a highly
non-convex objective.

We are interested in obtaining pricing policies for revenue
optimization in a data-rich (i.e., contextual) environment, where each
product is associated with a set of features. For example, in online
display advertising, a product is an ad impression (an ad placement
viewed by the user) which is annotated with features like
geo-information, device type, cookies, etc. There are two main
approaches to reserve pricing in this domain: one is to divide the
feature space into well defined clusters and apply a traditional
(non-contextual) revenue optimization algorithm in each
cluster~\cite{myerson1981optimal,dhangwatnotai2015revenue,roughgarden2016minimizing,paes2016field}.
This is effectively a semi-parametric approach with the drawback that
an overly fine clustering leads to data sparsity and inability to
learn across clusters. An overly coarse clustering, on the other hand,
does not fully take advantage of the rich features available.

To overcome these difficulties, a natural alternative is to fit a
parametric pricing policy by optimizing a loss function. The first
instinct is to use revenue itself as a loss function, but this loss is
notoriously difficult to optimize because it is discontinuous,
non-convex, and has zero gradient over much of its domain---so one
must look to surrogates. \citet{medina2014learning} propose a
continuous surrogate loss for revenue whose gradient information is
rich enough to optimize for prices. The loss is nevertheless
non-convex so optimizing it relies on techniques from constrained
DC-programming, which have provable convergence but limited
scalability in high-dimensional contexts.

\textbf{Main contribution.}
The main innovation in this paper is to address the revenue
optimization problem by instead looking to the closely related problem
of market clearing: how to set prices so that demand equals supply.
The loss function for market clearing exhibits several nice properties
from a learning perspective, notably convexity. The market clearing
objective dates back to the economic theory of market
equilibrium~\cite{arrowdebreu1954}, and more recently arises in the
literature on iterative
auctions~\cite{bikhchandani1997competitive,gul1999walrasian,ausubel2006efficient}.
To our knowledge, our work is the first to use it as a loss function
in a machine learning context.

The economic insight behind the market-clearing loss function allows us to adapt its
shape to control how conservative or aggressive the
resulting prices are in extracting revenue. To increase price levels,
we can artificially increase demand or limit supply, which connects
revenue optimization theorems from computational economics~\cite{bullow1996auction,roughgarden2012supply} to
regularization techniques under our loss function.

We begin by casting the problem of market clearing as a learning
problem. Given a dataset where each record corresponds to an item
characterized by a feature vector, together with buyer bids and seller
asks for the item, the goal of the pricing policy is to quote a price
that balances supply and demand; with a single seller, this simply
means predicting a price in between the highest- and second-highest
bids, which intuitively improves over the baseline of no reserve
pricing.


This offers us a general framework for price optimization in
contextual settings, but the objective function of market clearing is
still disconnected from revenue optimization. Revenue is the aggregate
price paid by buyers, while market clearing is linked to the problem
of optimizing efficiency (realized value). Efficiency can be measured
as social welfare (the total value of the allocated items), or more
coarsely via the match rate (the number of cleared transactions). The
platform faces a tension between trying to extract as much revenue as
possible from buyers, while also leaving them enough surplus to
discourage a move to competing platforms.

To better understand the trade-off between revenue and efficiency, we
consider the linear programming duality between allocation and pricing
and observe that a natural parameter that trades-off revenue for
efficiency is the available supply. Artificially limiting supply (or
increasing demand) allows one to control the aggressiveness of the
resulting clearing prices output by the model.
This fundamental idea has been used multiple times more recently in
algorithmic game theory to design approximately revenue-optimal
auctions~\cite{HR09,dhangwatnotai2015revenue,roughgarden2012supply,eden2017competition}.
Translating this intuition to our setting, a simple modification of
the primal (allocation) linear program has the effect of restricting
the supply. In the dual (pricing) linear program, this is equivalent
to adding a regularization to the market-clearing objective function.


The focus of this paper is empirical. 
As our main application, we use this methodology to optimize reserve
prices in display advertising auctions.
We demonstrate the efficacy of
the market clearing loss for reserve pricing by
experimentally comparing it with other strategies on a real-world data
set. Coupled with the experimental evaluation, we establish some
theoretical guarantees on match rate and efficiency for the optimal
pricing policy under clearing loss.
The theory provides guidance on how to set the regularization
parameters and we investigate how this translates to the desired trade-offs
experimentally.

\textbf{Experimental results.}
We evaluate our method against a linear-regression based approach on a
dataset consisting of over 200M auction records from a major display
advertising exchange. The features are represented as 84K-dimensional
sparse vectors and contain information such as the website on which
the ad will be displayed, device and browser type, and country of
origin.
As benchmarks we consider standard linear regression on either the
highest or second-highest bid, and models fit using the surrogate
revenue loss proposed by~\citet{medina2014learning}.
We find that our method Pareto-dominates the benchmarks in terms of
the trade-off between revenue and match rate or social welfare.
For example, for the best
revenue obtained from regression approach, we can obtain a pricing
function with at least the same revenue but 5\% higher social welfare
and 10\% higher match rate. 
%
%
We also find that the convergence rate of
fitting models under our loss function is as fast as a standard linear
regression. In comparison, the surrogate loss of
\citet{medina2014learning} has much slower convergence due to its
non-convexity.

\textbf{Related work.}
There is a large body of literature on
learning algorithms for optimizing revenue, however, most of the
literature deals with the non-contextual setting.
\citet{cole2014sample,morgenstern2015pseudo,morgenstern2016learning,paes2016field}
study the batch-learning non-contextual problem. \citet{roughgarden2016minimizing} study
the non-contextual problem both in the online and batch learning
settings. \citet{cesa2013regret} study it as a non-contextual online
learning problem. Finally, there has been a lot of recent interest in the
contextual online learning version
\cite{amin2014repeated,cohen2016feature,mao2018contextual}, but those
ideas are not applicable to the batch-learning setting.

Closest to our work are  \citet{medina2014learning} and
\citet{MedinaV17}, who also study contextual reserve price optimization in a
batch-learning setting. \citet{medina2014learning} proves generalization bounds,
defines a surrogate loss as a continuous approximation to the revenue loss, and
proposes an algorithm with provable convergence based on DC programming. The
algorithm, however, requires solving a convex program in each iteration.
\citet{MedinaV17} propose a clustering based approach, which involves the
following steps: learning a least-square predictor of the bid, clustering the
feature space based on the linear predictor, and optimizing the reserve using a
non-contextual method in each cluster.




\section{Market Clearing Loss}

This section introduces our model, proceeding from the general to the
specific. We first explain the duality between allocation and pricing,
which motivates the form of the loss function to fit clearing prices,
and provides useful economic insights into how the input data defines
its shape. We next define the formal problem of learning a clearing
price function in an environment with several buyers and sellers. We
then specialize to a single-item, second-price auction (multiple
buyers, single seller).

\subsection*{Allocation and Pricing}

We consider a market with $n$ buyers and $m$ sellers who aim to trade
quantities of an item (e.g., a stock or commodity) among themselves.
Each buyer $i$ is defined by a pair $(b_i, \mu_i)$ where $b_i \in \R_+$
is a bid price and $\mu_i \in \R_+$ is a quantity. The interpretation is
that the buyer is willing to buy up to $\mu_i$ units of the item at a
price of at most $b_i$ per unit. Similarly, each seller $j$ is defined
by a pair $(c_j, \lambda_j)$ where $c_j \in \R_+$ is an ask price and
$\lambda_j \in \R_+$ is the quantity of item the seller can supply. The ask
price can be viewed as a cost of production, or as an outside offer
available to the seller, so that the seller will decline to sell item units
for any price less than its ask.

The allocation problem associated with the market is to determine
quantities of the item supplied by the sellers, and consumed by the
buyers, so as to maximize the \emph{gains from trade}---value consumed
minus cost of production. Formally, let $x_i$ be the quantity bought
by buyer $i$ and $y_j$ the quantity sold by seller $j$. The optimal
gains from trade are captured by the (linear) optimization problem:
\begin{eqnarray}
  \max_{0 \leq x_i \leq \mu_i, 0 \leq y_j \leq \lambda_j}
  & \ds \sum_{i=1}^n b_i x_i - \sum_{j=1}^m c_j y_j & \nonumber \\
\mbox{s.t.} &  \ds \sum_{i=1}^n x_i = \sum_{j=1}^m y_j & \label{cons:supply-demand}
\end{eqnarray}
The optimization is straightforward to solve: the highest bid is
matched with the lowest ask, and the two agents trade as much as
possible between each other. The process repeats until the highest bid
falls below the lowest ask.
%
The purpose of the linear programming formulation is to consider its
dual, which corresponds to a pricing problem:
%
%
\begin{equation}
  \label{eq:clearing-loss}
\min_p \:\: \sum_{i=1}^n \mu_i\relu{b_i - p} + \sum_{j=1}^m \lambda_j\relu{p - c_j}
\end{equation}
where $\relu{\cdot}$ denotes $\max\{\cdot, 0\}$. The optimal dual
solution corresponds to a price that balances demand and
supply, which is the central concept in this paper.
\begin{definition}
\label{def:clearing-price}
  A price $p^*$ is a \emph{clearing price} if, for any optimal
  solution $(\bm{x}^*,\bm{y}^*)$ to the allocation problem, we have
  \begin{eqnarray*}
    x_i^* & \in & \argmax_{x_i \in [0, \mu_i]} \: x_i(b_i - p) \\
    y_j^* & \in & \argmax_{y_j \in [0, \lambda_j]} \: y_j(p - c_j)
  \end{eqnarray*}
  for each buyer $i$ and seller $j$,
\end{definition}
In words, a clearing price balances supply and demand by ensuring
that, at an optimal allocation, each buyer buys a quantity that
maximizes its utility (value minus price), and similarly each seller
sells a quantity that maximizes its profit (price minus cost). In the
current simple setup with a single item, buyer $i$ will buy $\mu_i$
units if $b_i > p$, zero units if $b_i < p$, and is indifferent to the
number of units bought at $p = b_i$; similarly for each seller $j$.
However, the concept of clearing prices---where each agent maximizes
its utility at the optimal allocation---generalizes to much more
complex allocation problems with multiple differentiated items and
nonlinear valuations over bundles of
items~\cite{bikhchandani1997competitive}.

The fact that a clearing price exists, and can be obtained by
solving~(\ref{eq:clearing-loss}), follows from standard LP duality.
The complementary slackness conditions relating optimal primal
solution $(\bm{x}^*,\bm{y}^*)$ to optimal dual solution $p^*$ amount
to the conditions of Definition~\ref{def:clearing-price}.
The optimal solution $p^*$ to the dual corresponds to a Lagrange
multiplier for constraint~(\ref{cons:supply-demand}) which equates
demand and supply.

\subsection*{Learning Formulation}

To cast market clearing in a learning context, we consider a generic
feature space $\fspace$ with the label space $\tspace = \R_+^n \times
\R_+^m$ consisting of bid and ask vectors $(\bids, \asks)$.
For the sake of simplicity, we develop our framework assuming that the
number of buyers and sellers remains fixed (at $n$ and $m$), and that
the item quantity that each agent demands or supplies ($\mu_i$ or
$\lambda_j$) is also fixed. This information is straightforward to
incorporate into the label space if needed, and our results can be
adapted accordingly. The objective is to fit a price predictor (also
called a pricing policy)
$p : \fspace \rightarrow \R$ to a training set of data
$\{(z_k, \bids_k, \asks_k)\}$ drawn from
$\fspace \times \tspace$, to achieve good prediction performance on
separate test data drawn from the same distribution as the training
data.

As a concrete example, the training data could consist of bids and
asks for a stock on a financial exchange throughout time, and the
features might be recent economic data on the company, time of day or
week, etc. The clearing problem here is equivalent to predicting a
price within each datapoint's bid-ask spread given the features. As
another example, the data could consist of bids for ad impressions on
a display ad exchange, and the features might be contextual
information about the website (e.g., topic) and user (e.g., whether
she is on mobile or desktop). The clearing problem there reduces to
predicting a price between the highest and second-highest bids.

Based on our developments so far, the correct loss function to fit
clearing prices is given by~(\ref{eq:clearing-loss}), which we call
the \emph{clearing loss}:
\begin{equation*}
  \label{clearing-loss}
\ell^c(p, z, \bids, \asks) = \sum_{i=1}^n \mu_i\relu{b_i - p(z)} + \sum_{j=1}^m \lambda_j\relu{p(z) - c_j}
\end{equation*}
Figure~\ref{fig:clearing-loss} illustrates the shape of the clearing
loss (in green) under an instance with buyers
$(\$1, 1),\, (\$4, 1),\, (\$5, 2)$ and sellers $(\$2, 1),\, (\$3, 1)$.
Note that although the first buyer's bid of $\$1$ lies below any of
the sellers' costs, it still contributes to the shape of the loss.
Here any price between $\$4$ and $\$5$ is a clearing price. If we add
an extra buyer $(\$6, 1)$, the loss curve tilts to the right (in blue)
and the unique clearing price becomes $\$5$; since there is more
demand, the clearing price increases. If we instead add an extra
seller $(\$2, 2)$, the curve tilts to the left (in pink) and the
clearing price decreases; now any price between $\$3$ and $\$4$ is a
clearing price. This example hints at a way to control the
aggressiveness of the price function $p$ fit to the data, by
artificially adjusting demand or supply.
\begin{figure}[t!]
  \begin{center}
    \includegraphics[width=0.8\textwidth]{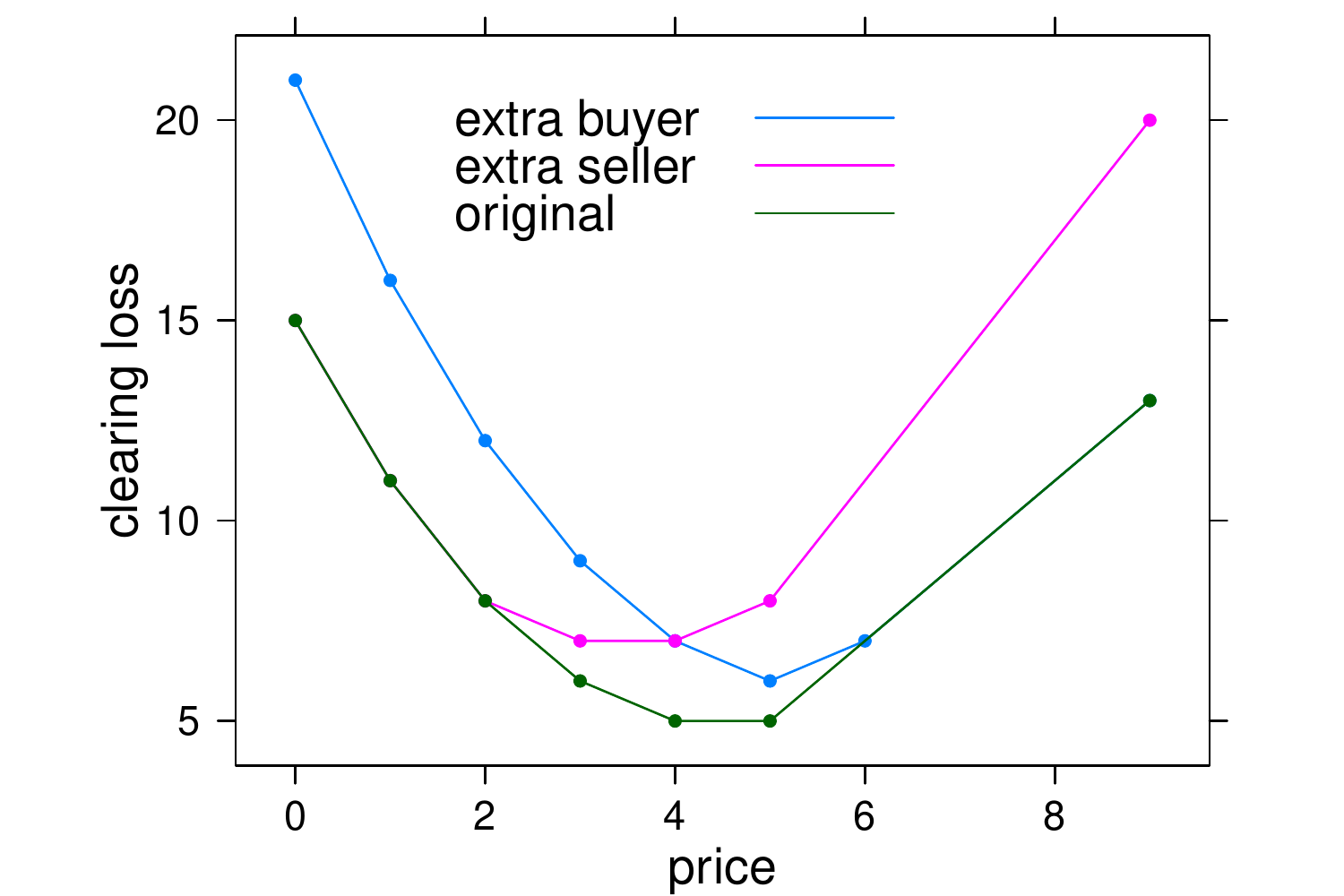}
  \end{center}
  \caption{Effect on the shape of the clearing loss when adding a
    buyer or a seller.}\label{fig:clearing-loss}
\end{figure}

Over a training set of data $\{(z_k, \bids_k, \asks_k)\}$, model
fitting consists of computing a pricing policy $p$ that minimizes the
overall loss $\sum_k \ell^c(p, z_k, \bids_k, \asks_k)$. Under a
limited number of contexts $z_k$, it may be possible to directly
compute optimal clearing prices, or even revenue-maximizing
reserve prices, based on the bid distributions in each
context~\cite{cole2014sample,myerson1981optimal}. But this kind of
nonparametric approach quickly runs into difficulties when there is a
large number of contexts or even continuous features, where issues of
data sparsity and discretization arise. Our formulation allows one to
impose some structure on the pricing policy (e.g., a linear model or
neural net) whenever this aids with generalization.

From a learning perspective, clearing loss has several attractive
properties. It is a piece-wise linear convex function of the price,
where the kink locations are determined by the bids and asks. The
magnitude of its derivatives depends only on the buyer and seller
quantities, which makes it robust to any outliers in the bids or asks.
By its derivation via LP duality, its optimal value equals the optimal
gains from trade, which are easy to compute. This gives a reference
point to quantify how well a price function fits any given dataset.

\subsection*{Reserve Pricing}

As a practical application of the clearing loss, we consider the
problem of reserve pricing in a single-item, second-price auction. In
this setting every buyer demands a single unit ($\mu_i = 1$), and
there is a single seller ($m = 1$) with cost $c$. The seller also has
unit supply, but we still parametrize its quantity by $\lambda$ to
allow some control on the shape of the loss.

We write $\hb$ and $\hob$ to denote the highest and second-highest
bids, respectively. In a single-item second-price auction, the item is
allocated to the highest bidder as long as $\hb \geq c$, and is
charged $\secp \equiv \max\{\hob, c\}$.
Second-price auctions are extremely common and until now have been the
dominant format for selling display ads online through ad exchanges,
among countless other applications.
It is common in second-price auctions for the seller to set a
\emph{reserve price}, a minimum price that the winning bidder is
charged. The cost $c$ is itself a reserve price, but the seller may
choose to increase this to some price $p$ in an attempt to extract
more revenue, at the risk of leaving the item unsold if it turns out
that $\hb < p$. Revenue as a function of $p$ can be negated to define
a loss, which we denote $\ell^r$:
$$
-\ell^r(p, z, \bids, c) = \left\{
  \begin{array}{cl}
    \max\{p(z), \secp\} & \mbox{if $\max\{p(z), c\} \leq \hb$} \\
    c & \mbox{otherwise} \\
  \end{array}
\right.
$$
However, this loss is notoriously difficult to optimize directly,
because it is non-convex and even discontinuous, and its gradient is 0
except over a possibly narrow range between the highest and
second-highest bids. Clearing loss represents a promising alternative
for reserve pricing because any price between $\secp$ and
$\hb$ is a clearing price, so a correct clearing price prediction
should intuitively improve over the baseline of $c$. The clearing loss
in the auction setting takes the form:
\begin{equation}
  \label{clearing-loss-auction}
\ell^c(p, z, \bids, c) = \sum_{i=1}^n \relu{b_i - p(z)} + \lambda\relu{p(z) - c}
\end{equation}

In practical applications of reserve pricing it is often desirable to
achieve some degree of control over the \emph{match rate}---the
fraction of auctions where the item is sold---and the closely related
metric of \emph{social welfare}---the aggregate value of the items
sold, where value is captured by the winning bid $\hb$. Formally, these
concepts are defined as follows, where the notation $\pred{\cdot}$ is
1 if its predicate is true and 0 otherwise.
\begin{definition}
  On a single data point, the \emph{match rate} at price $p$ is
  $\mr(p) = \pred{\hb \geq \max\{p,c\}}$ and the \emph{social welfare} is
  $\sw(p) = \hb\pred{\hb \geq \max\{p,c\}}$.  
\end{definition}
As with the revenue objective, match rate and social welfare are
discontinuous and their gradients are almost everywhere 0, so they are
not directly suitable for model fitting via convex optimization (i.e.,
one has to look to surrogates).

Note that the clearing loss~(\ref{clearing-loss-auction}) effectively contains
a term that approximately regularizes according to match rate. The
seller's term $\relu{p - c}$ can be viewed as a hinge-type surrogate
for match rate, since any setting of $p$ above $c$ risks impacting
match rate. Increasing $\lambda$ improves match rate, in line with the
earlier economic intuition that increasing seller supply $\lambda$
shifts the clearing price downwards. Symmetrically, $\lambda$ can be
decreased within the range $[0,1]$ (the loss remains convex in this
range), which is equivalent to increasing each buyer's demand to
$\mu = 1/\lambda$. According to the economic intuition, this shifts
the clearing price upwards at the expense of match rate. The fact that
the relevant range and units of the regularization weight
$\lambda$ are understood is very convenient in practice. In the next
section, we derive a quantitative link between $\lambda$
and match rate.


\section{Theoretical Guarantees}

In this section we prove approximation guarantees on the match rate
and efficiency performance of models fit using the clearing loss.
The results of this analysis will provide guidelines for setting the
regularization parameters for fine-grained control of the match rate.
%
%

We begin by characterizing the optimal pricing policy under clearing
loss when there is no restriction on the policy structure, assuming
that bids and costs are drawn independently (but not necessarily
identically).
\begin{prop} If conditioned on each feature vector $z$ the bid and cost
  distributions are given by $b_i \sim F_i^z$ and $c_j \sim G_j^z$, then
  the pricing policy that optimizes clearing loss is the solution to
  $$\sum_i \mu_i (1-F_i^z(p(z))) = \sum_j \lambda_j G_j^z(p(z)),$$
  which is the policy that balances expected supply and demand.
\end{prop}
\begin{proof}
We can write the expectation of the market clearing loss function as follows:
  $$\begin{aligned} \E[\ell^c(p)] = & \sum_{i=1}^n \mu_i \int_{p}^{\infty} (b_i-p)
    \,\mathrm{d}F_i^z(b_i) \\
  + & \sum_{j=1}^m
  \lambda_j \int_0^p (p-c_j) \,\mathrm{d}G_j^z(c_j). \end{aligned}$$
  Taking the derivative with respect to $p$ and setting it to zero leads to the
  result in the statement:
  $$0 = \frac{\mathrm{d}}{\mathrm{d}p} \E[\ell^c(p)] = - \sum_{i=1}^n  \mu_i (1-F_i^z(p)) + \sum_{j=1}^m
  \lambda_j G_j^z(p).$$
\end{proof}
We now consider the single-item auction setting where $m = 1$ and
$\mu_i = 1$ for all buyers. For simplicity, also assume that $c=0$,
which implies $G_j(p) = 1$ for all $p$. In that case we can bound the
match rate by a simple formula.

\begin{prop} 
  \label{prop:mr-bound}
  In the setup with a single seller with $\lambda$ supply and cost
  $c = 0$, and independent buyer distributions, the expected match
  rate under the optimal clearing price policy is at least
  $1-e^{-\lambda}$.
\end{prop}

\begin{proof}
  A transaction clears if there is at least one buyer with valuation above the
  price $p$ which happens with probability $1-\prod_{i=1}^n F_i^z(p)$. Since the
  optimal policy $p$ is the solution of $\sum_{i=1}^n (1-F_i^z(p)) = \lambda$ by
  the previous proposition, we can bound the match rate as follows:
  $$\begin{aligned}
    \E[\mr] & = 1-\prod_{i=1}^n F_i^z(p)  \geq 1-\left[ \frac{1}{n}
    \sum_{i=1}^n F_i^z(p)
    \right]^n \\ &= 1-\left[ 1 - \frac{\lambda}{n} \right]^n \geq 1-e^{-\lambda}
  \end{aligned}$$
  where the first inequality follows from the arithmetic-geometric mean
  inequality.
\end{proof}

The preceding proposition provides a useful guideline on how to set the
regularization parameter $\lambda$ to achieve a certain target match
rate. We can also obtain a similar bound for social welfare:
\begin{cor} In the setting of the previous proposition, the social
  welfare $\E[\sw] = \E[\hb \cdot \pred{\hb \geq p}]$ obtained by the optimal
  clearing price policy is at least $1-e^{-\lambda}$ of the optimal
  social welfare, obtained by setting no reserves.
\end{cor}
\begin{proof}
  This follows from the fact that $\E[\hb \cdot \pred{\hb \geq p}] \geq
  \E[\hb] \cdot \P[\hb \geq p] \geq (1-e^{-\lambda}) \cdot \E[\hb]$.
\end{proof}

Another interesting corollary is that when buyers are i.i.d., fitting a clearing price is equivalent to fitting a certain quantile of the common bid distribution.
\begin{cor} In the setup of the previous proposition with i.i.d.\
  buyers, the optimal clearing price policy is to set the price at
  $p(z) = F^{-1}(1-\lambda/n)$ where $F = F_i^z$.
\end{cor}
This result makes explicit how varying $\lambda$ in the clearing loss
tunes the aggressiveness of the resulting price function, by moving up
or down the quantiles of the bid distribution. In particular, it's
possible to span all quantiles using $\lambda \in [0,+\infty]$.
Fitting clearing prices is not exactly equivalent to quantile
regression, since the relevant quantile depends on the number of
buyers, which is a property of the data and not fixed in advance.



\section{Empirical Evaluation}

In this section we evaluate our approach of using predicted clearing
prices as a reserve pricing policy in second-price auctions. We
collected a dataset of auction records by sampling a fraction of the
logs from Google's Ad Exchange over two consecutive days in
January 2019. Our sample contains over 100M records for each day. In
display advertising, online publishers (e.g., websites like
nytimes.com) can choose to request an ad from an exchange when a user
visits a page on their site. The exchange runs a second-price auction
(the most common auction format) among eligible advertisers, possibly
with a reserve price.

We clip bid vectors to the 5 highest bids. As the publisher cost $c$,
we use a reserve price available in the data which is meant to capture
the opportunity cost\footnote{A common alternative source
  of display ads besides exchanges are reservation contracts, which
  are advertiser-publisher agreements to show a fixed volume of ads
  for a time period. If the contract is not fulfilled, this comes at a
  penalty to the publisher.}
of not showing ads from other sources besides the exchange, in line
with our model.\footnote{We also excluded additional sources of reserve
  prices from the dataset: (a) reserve prices configured by
  publishers reflecting business objectives like avoiding channel conflict
  (i.e., protecting the value of inventory sold through other means)
  and (b) automated reserve prices set by the exchange.}
Reserve prices are only relevant conditional on the top bid exceeding
the publisher cost, so the auction records were filtered to satisfy
this condition. When reporting our results this means that the
baseline match rate without any reserve pricing is 100\%, so we will
refer to it as \emph{relative} match rate in our plots to emphasize
this fact.

All the models we evaluate\footnote{We evaluate the models by
  simulating the effect of the new reserves on a test dataset. The
  simulation does not take into account possible strategic responses
  on the part of buyers. However, since the auction format is a second
  price auction, it is a dominant strategy for the buyers to bid
  truthfully.} are linear models of
the price $p$ as a function of features $z$ of the auction records.
The only difference between the models is the loss function used to
fit each one, to focus on the impact of the choice of loss function.
The features we used included: publisher id, device type (mobile,
desktop, tablet), OS type (e.g., Android or iOS), country, and format
(video or display). For sparse features like publisher id we used a
one-hot encoding for the most common ids and an `other' bucket for ids
in the tail. The models were all fit using TensorFlow with the default
Adam optimizer and minibatches of size 512 distributed over 20
machines. An iteration corresponds to one minibatch update in each
machine, therefore $20 \times 512$ data points. The models were all
trained over at least 400K iterations, although for some models convergence
occurred much earlier.

Besides the clearing loss used to fit our model, we considered several
other losses as benchmarks:
\begin{itemize}
\item Least-squares regression on the highest bid $\hb$.
\item Least-squares regression on the 2nd-highest bid $\hob$.
\item A revenue surrogate loss function proposed by~\citet{medina2014learning} as a
  continuous alternative to the pure revenue loss $\ell^r$ mentioned
  previously:
$$
-\ell^\gamma(p, z, \bids,c) =
\left\{
    \begin{array}{ll}
      \max\{p(z),\secp\} & \hspace{-15pt}\mbox{if $p(x) \leq b_1$} \\
       c & \hspace{-45pt}\mbox{if $p(x) > (1+\gamma)b_1$} \\
      ((1+\gamma)b_1 - p(x))/\gamma  & \hspace{-5pt}\mbox{otherwise}\\
    \end{array}
    \right.
$$
The loss has a free parameter $\gamma > 0$ which can be tuned to
control the approximation to $\ell^r$. Although this loss is
continuous, it is still non-convex. In our experiments we tried a
range of $\gamma \in \{0.25, 0.5, 0.75, 1\}$. Below we report on the
setting $\gamma = 0.75$ which gave the best revenue performance.
\end{itemize}

For each loss function we added the match-rate regularization
$\lambda \relu{p - c}$, and we varied $\lambda$ to span a range of
realized match rates. Recall that this regularization is already
implicit in the clearing loss, where $\lambda$ can be construed as the
item quantity supplied by the seller. We used non-negative $\lambda$
to ensure that convexity is preserved if the original loss is itself
convex.

We used the first day of data as the training set and the second day
as the test set.
The performance was very similar on both for all fitted models, which is
expected due to the volume of data and the generalization properties
of this learning problem \cite{morgenstern2015pseudo}.
We report results over the test set below.

\subsection*{Revenue Performance}

\begin{figure*}[th!]
  \centering
  \begin{subfigure}[h]{0.8\textwidth}
  	\centering
    \includegraphics[width=\textwidth]{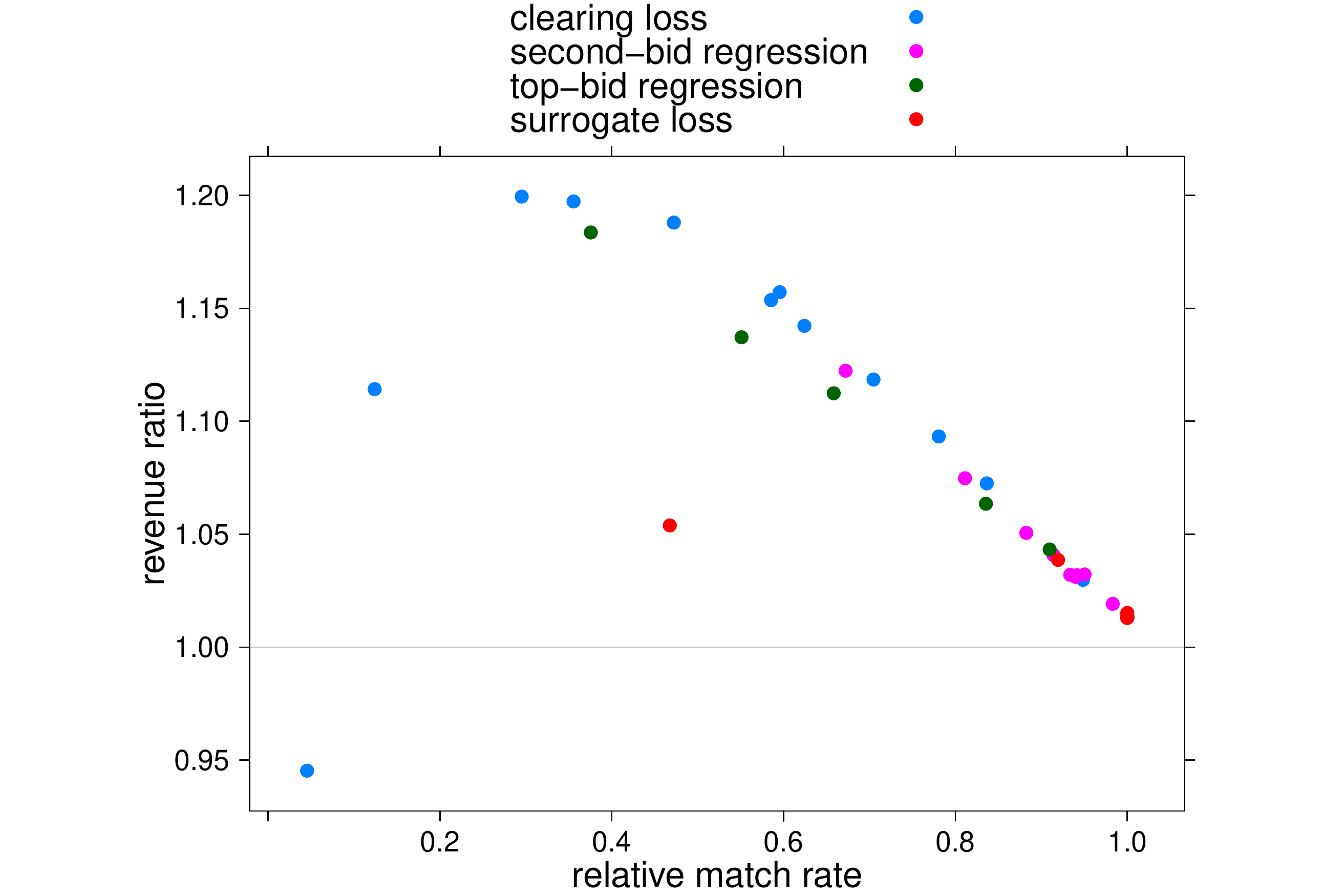}
  \end{subfigure}
\vskip 15pt
  \begin{subfigure}[h]{0.8\textwidth}
  	\centering
    \includegraphics[width=\textwidth]{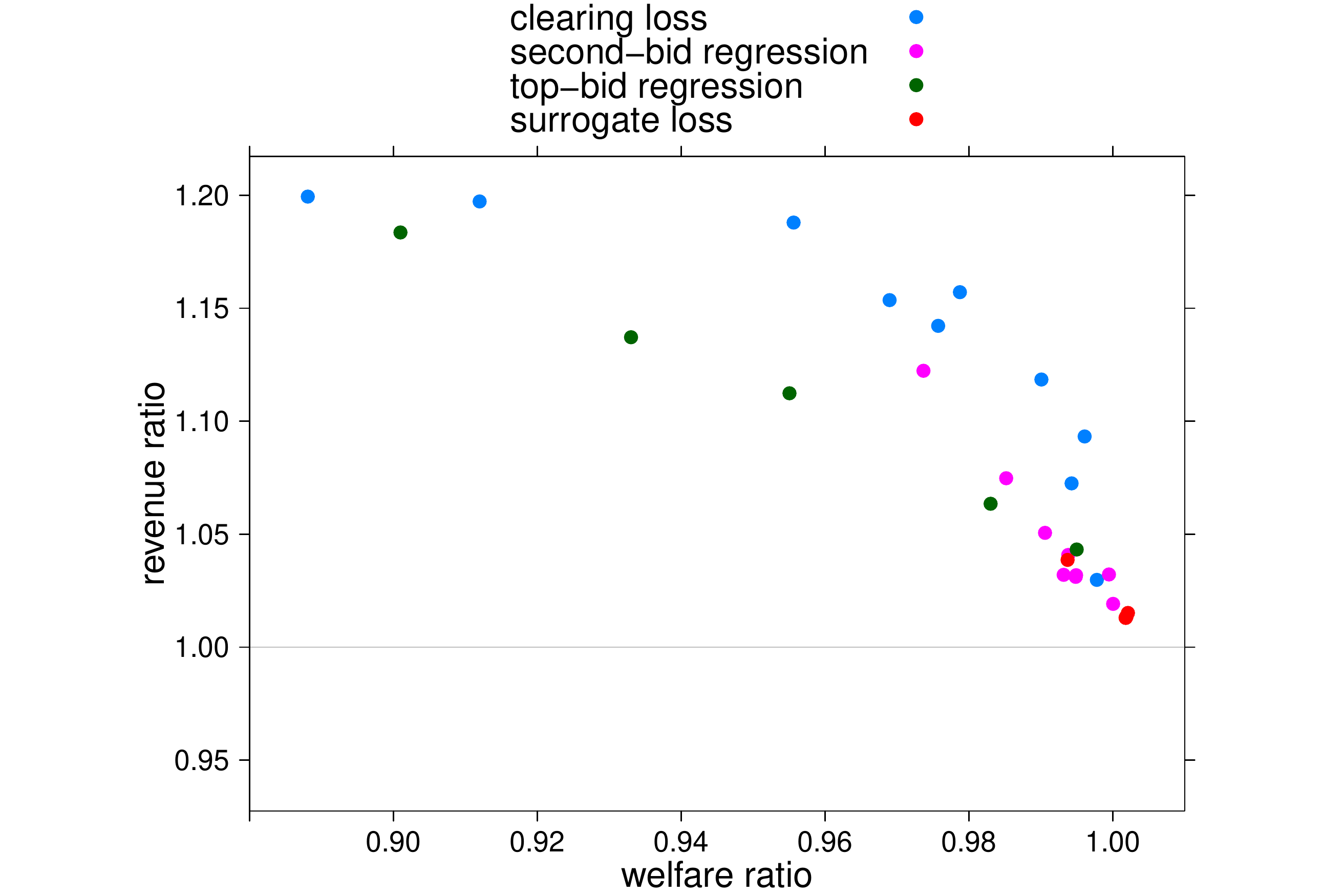}
  \end{subfigure}
  \caption{Trade-off between revenue improvement and decrease in match
    rate (top) or buyer welfare (bottom). Each point represents the
    performance of the fitted model under a loss function for a fixed
    regularization level.}
  \label{fig:revenue-ratio-tradeoff}
\end{figure*}

We first consider the revenue performance of the different losses as
it trades off against match rate and buyer welfare.
Figure~\ref{fig:revenue-ratio-tradeoff} plots the ratio of realized
revenue with learned reserves against the realized match rate (top).
Both axes are normalized by the revenue and match rate of the
second price auction using only the seller's cost as reserves. Each
point represents a pair of revenue and match rate or welfare achieved
at a certain setting $\lambda$. The most immediate observation is that
the curve traced out by the clearing loss Pareto dominates the
performance of the benchmark loss functions, in the sense that for any
fixed match rate, the clearing loss' revenue performance lies higher
than the others. The best revenue performance is a $20\%$ improvement
achieved by the clearing loss at $\lambda = 0.25$ with a match rate of
$30\%$.

We also plot in the figure the revenue against welfare (bottom). We
again normalize each axis by the revenue and welfare of the auction
that uses only seller's costs (which achieves the optimal social
welfare). For the sake of clarity the range of the x-axis has been
clipped. The Pareto dominance here is even more pronounced, and it's
also striking to note that clearing loss can achieve revenue
improvements of over $10\%$ with less than $2\%$ impact on buyer
welfare.

Another interesting aspect of Figure~\ref{fig:revenue-ratio-tradeoff}
is the range of match rates spanned by the different losses. Recall
that, under the assumptions and results of
Proposition~\ref{prop:mr-bound}, varying $\lambda$ from 0 to large
values should allow the clearing loss to span the full range of match
rates in $(0,1)$, and this is borne out by the plot. For the
regressions on $\hb$ and $\hob$, there is a hard floor on the match
rate that they can achieve with $\lambda = 0$, respectively at 0.38
and 0.67. Another kind of regularization term would be needed to push
these further downward and reach more aggressive prices. Match rate
for the surrogate loss was particularly sensitive to regularization.
Over a range of $\lambda$ spanning from 0 to 1, only $\lambda =0$ and
$\lambda = 0.1$ yielded match rates below 1, at 0.47 and 0.92
respectively.

\subsection*{Controlling Match Rate}

In practice setting the right regularization weight $\lambda$ to
achieve a target match rate is usually process of trial and error, even to
determine the relevant range to inspect, and this was the case for all
the benchmark losses. For the clearing loss, however,
Proposition~\ref{prop:mr-bound} gives a link between match rate and
$\lambda$ which can serve as a guide. Specifically, the result
prescribes $\lambda = \log(\frac{1}{1-\mr})$ to achieve a match rate
of $\mr$.

\begin{figure}[h!]
	\centering
    \includegraphics[width=0.8\textwidth]{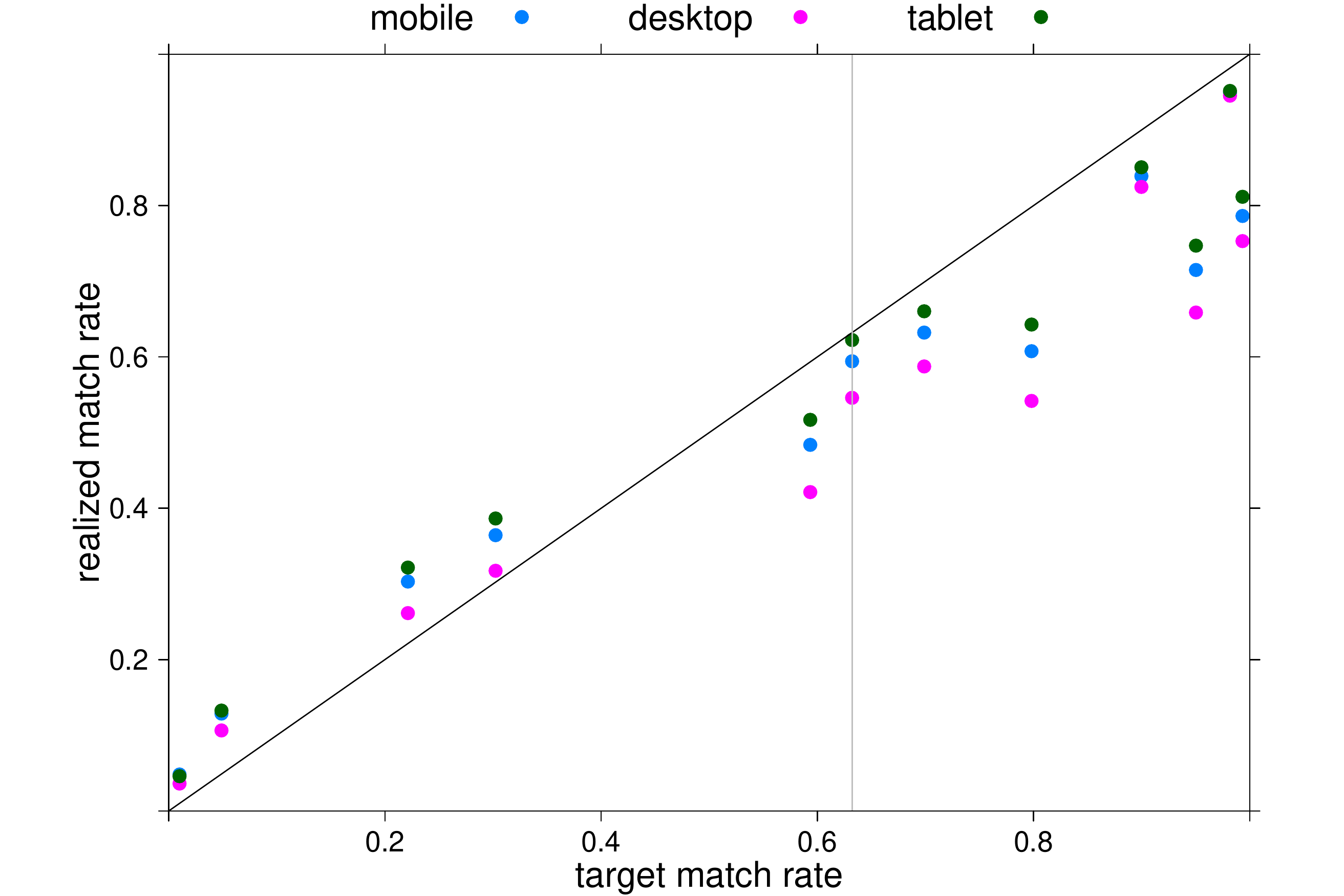}
  \caption{Realized match rate against target match rate under the
    model fit with the clearing loss, broken down by device type. The
    vertical line denotes the parameter setting $\lambda = 1$ with a
    target match rate of $1 - 1/e \approx 0.63$.}
  \label{fig:target-match-rate}
\end{figure}

Figure~\ref{fig:target-match-rate} plots the target match rate implied
by the settings of $\lambda$ that we used, according to this
formula, against realized match rates. The vertical line shows the
reference point of $\lambda = 1$, which is the ``default'' form of the
clearing loss without artificially increasing or limiting supply, with
an associated match rate $1-1/e \approx 0.63$. The realized match rate
tracks the target fairly well but not perfectly. A possible reason for
the discrepancy is that the assumption of i.i.d.\ bidders that the
formula relies on may not hold in practice. Another possible reason is that the linear model may not be
expressive enough to fit the optimal price level within each feature
context $z$.
Interestingly, the target match rate from Proposition~\ref{prop:mr-bound} tracks
not only the overall match-rate but also segment-specific match rate. In
Figure~\ref{fig:target-match-rate}, we break down the match rates by device
type and find that they are very consistent across devices.

\subsection*{Convergence Rate}

We next consider the convergence rates of model-fitting under the
various loss functions, plotted in Figure~\ref{fig:convergence-rate}.
Convergence rates for the clearing loss and the regression losses are
very comparable. The main difference between the curves has to do with
initialization. Initial prices tended to be high under our random
initialization scheme, which is more favorable to regression on the
highest bid. All models have converged by 100K iterations.
Since square loss is ideal from an optimization perspective, these
results imply that models with clearing loss can be fit very quickly
and conveniently in practice, in a matter of hours over large display
ad datasets.

\begin{figure}[t!]
	\centering
    \includegraphics[width=0.7\textwidth]{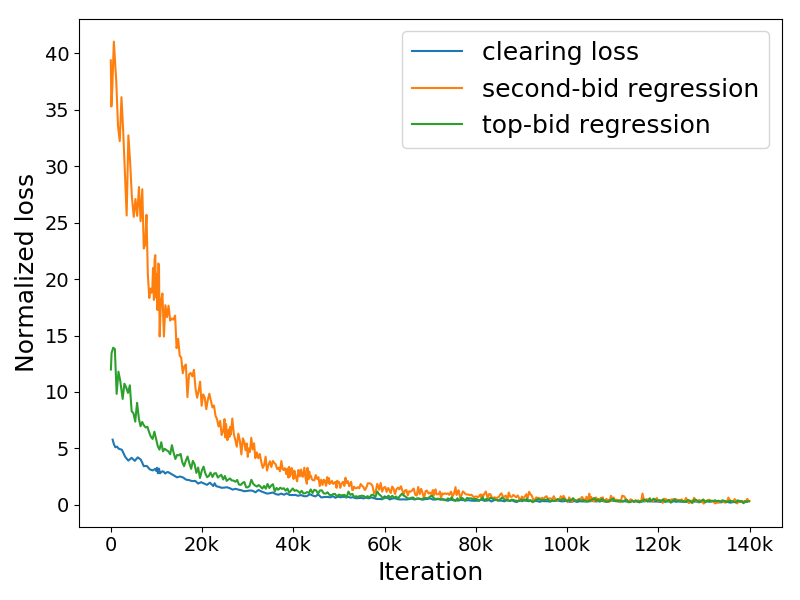}
  \caption{Convergence rate of the model under different loss
    function, in minibatch iterations. We plot the value of each loss
    across iterations normalized by its value upon convergence.}
  \label{fig:convergence-rate}
  \vspace{-12pt}
\end{figure}

In Figure \ref{fig:convergence-rate-surrogate} we compare the
convergence of the clearing loss with the surrogate loss. Convergence
is much slower under surrogate loss. This was expected, as the loss is
nonconvex and it has ranges with 0 gradient where the Adam optimizer
(or any of the other standard TensorFlow optimizers) cannot make
progress; it was nonetheless an important benchmark to evaluate since
it closely mimics the true revenue objective.
\citet{medina2014learning} discuss alternatives for optimizing the
surrogate loss, and propose a special purpose algorithm based on
DC-programming (difference of convex functions programming), but they
only scale it to thousands of training instances.
The fact that the surrogate loss has not quite converged after 400K
iterations is a contributing factor to its revenue performance in
Figure~\ref{fig:revenue-ratio-tradeoff}.

\subsection*{Effectiveness of Linear Regression}

While the key take-away of our empirical evaluation is the fact that
the clearing loss dominates other methods in terms of revenue vs.\
match rate trade-offs, another surprising consequence of this study is
the effectiveness of using a simple regression on the top bid. The
natural intuition would be that any least-squares regression should perform
poorly since it has the same penalty for underpricing (which is a small
loss in revenue) and overpricing (which can cause the transaction to fail
to clear and hence incur in a large revenue loss). Indeed it is the
case that an unregularized regression (the leftmost green point on
Figure~\ref{fig:revenue-ratio-tradeoff}) incurs a large match rate
loss, but it also achieves significant revenue improvement (albeit
with an almost 5\% loss in social welfare compared to the clearing
loss).

Looking into the data, we found that an explanation for this fact is
that bid distributions tend to be highly skewed which causes
standard regression to underpredict for high bids and overpredict for
low bids. In fact, under zero regularization the linear regression on
the top bid underpredicts 17.7\% of instances for bids below the median
and 99.1\% for bids above the median. This type of behavior explains
why standard regression can be effective in practice despite the fact
that square loss does not
encode any difference between underpredicting and overpredicting.

\begin{figure}[t!]
	\centering
    \includegraphics[width=0.7\textwidth]{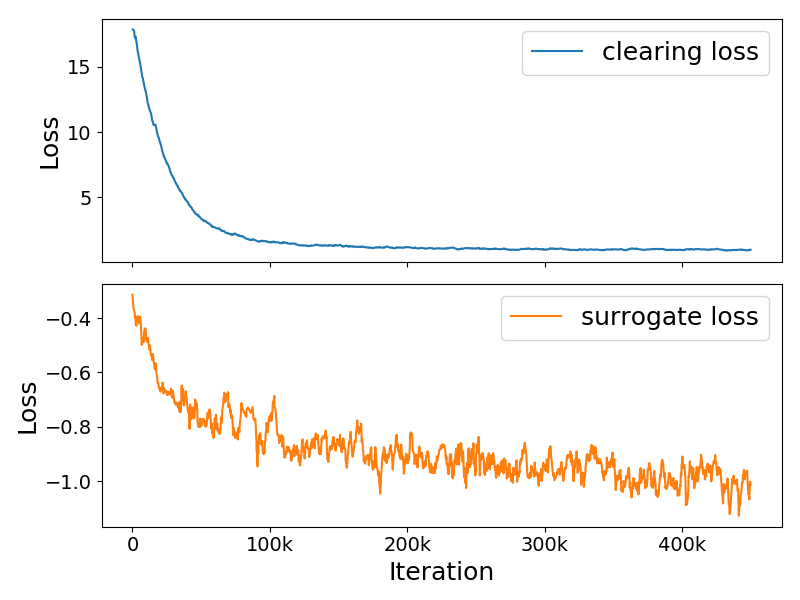}
  \caption{Convergence rate of the model under clearing and surrogate
    function, in minibatch iterations. Both loss functions are
    smoothed using a $0.9$ moving average.}
  \label{fig:convergence-rate-surrogate}
  \vspace{-12pt}
\end{figure}

\section{Conclusions}

This paper introduced the notion of a predictive model for clearing
prices in a market with bids and asks for units of an item. The loss
function is obtained via the linear programming dual of the associated
allocation problem. When applied to the problem of revenue
optimization via reserve prices in second-price auctions, regularizing
the loss has an intuitive interpretation as expanding or limiting
supply, which can be formally linked to the expected match rate. Our
empirical evaluation over a dataset of bids from Google's Ad
Exchange confirmed that a model of clearing prices
outperforms standard regressions on bids, as well as a surrogate loss for the
direct revenue objective, in terms of the trade-off between revenue
and match rate (or social welfare).
%
%
In future work, we plan to develop models of
clearing prices for more complex allocation problems such as search
advertising, where the clearing loss can be generalized (using the same
duality ideas presented in this paper) to handle a vector of position
prices.

\bibliography{pricing}

\begin{thebibliography}{21}
\providecommand{\natexlab}[1]{#1}
\providecommand{\url}[1]{\texttt{#1}}
\expandafter\ifx\csname urlstyle\endcsname\relax
  \providecommand{\doi}[1]{doi: #1}\else
  \providecommand{\doi}{doi: \begingroup \urlstyle{rm}\Url}\fi

\bibitem[Amin et~al.(2014)Amin, Rostamizadeh, and Syed]{amin2014repeated}
Kareem Amin, Afshin Rostamizadeh, and Umar Syed.
\newblock Repeated contextual auctions with strategic buyers.
\newblock In \emph{Advances in Neural Information Processing Systems}, pages
  622--630, 2014.

\bibitem[Arrow and Debreu(1954)]{arrowdebreu1954}
Kenneth~J. Arrow and Gerard Debreu.
\newblock Existence of an equilibrium for a competitive economy.
\newblock \emph{Econometrica}, 22\penalty0 (3):\penalty0 265--290, 1954.

\bibitem[Ausubel(2006)]{ausubel2006efficient}
Lawrence~M Ausubel.
\newblock An efficient dynamic auction for heterogeneous commodities.
\newblock \emph{American Economic Review}, 96\penalty0 (3):\penalty0 602--629,
  2006.

\bibitem[Bikhchandani and Mamer(1997)]{bikhchandani1997competitive}
Sushil Bikhchandani and John~W Mamer.
\newblock Competitive equilibrium in an exchange economy with indivisibilities.
\newblock \emph{Journal of Economic Theory}, 74\penalty0 (2):\penalty0
  385--413, 1997.

\bibitem[Bulow and Klemperer(1996)]{bullow1996auction}
Jeremy Bulow and Paul Klemperer.
\newblock Auctions versus negotiations.
\newblock \emph{American Economic Review}, 86:\penalty0 180--194, 1996.

\bibitem[Cesa-Bianchi et~al.(2013)Cesa-Bianchi, Gentile, and
  Mansour]{cesa2013regret}
N.~Cesa-Bianchi, C.~Gentile, and Y.~Mansour.
\newblock Regret minimization for reserve prices in second-price auctions.
\newblock In \emph{ACM-SIAM Symposium on Discrete Algorithms (SODA)}, pages
  1190--1204. SIAM, 2013.

\bibitem[Cohen et~al.(2016)Cohen, Lobel, and Paes~Leme]{cohen2016feature}
Maxime~C Cohen, Ilan Lobel, and Renato Paes~Leme.
\newblock Feature-based dynamic pricing.
\newblock In \emph{Proceedings of the 2016 ACM Conference on Economics and
  Computation}, pages 817--817. ACM, 2016.

\bibitem[Cole and Roughgarden(2014)]{cole2014sample}
Richard Cole and Tim Roughgarden.
\newblock The sample complexity of revenue maximization.
\newblock In \emph{Proceedings of the 46th annual ACM Symposium on Theory of
  Computing}, pages 243--252. ACM, 2014.

\bibitem[Dhangwatnotai et~al.(2015)Dhangwatnotai, Roughgarden, and
  Yan]{dhangwatnotai2015revenue}
Peerapong Dhangwatnotai, Tim Roughgarden, and Qiqi Yan.
\newblock Revenue maximization with a single sample.
\newblock \emph{Games and Economic Behavior}, 91:\penalty0 318--333, 2015.

\bibitem[Eden et~al.(2017)Eden, Feldman, Friedler, Talgam-Cohen, and
  Weinberg]{eden2017competition}
Alon Eden, Michal Feldman, Ophir Friedler, Inbal Talgam-Cohen, and S~Matthew
  Weinberg.
\newblock The competition complexity of auctions: A {B}ulow-{K}lemperer result
  for multi-dimensional bidders.
\newblock In \emph{Proceedings of the 2017 ACM Conference on Economics and
  Computation}, pages 343--343. ACM, 2017.

\bibitem[Gul and Stacchetti(1999)]{gul1999walrasian}
Faruk Gul and Ennio Stacchetti.
\newblock Walrasian equilibrium with gross substitutes.
\newblock \emph{Journal of Economic Theory}, 87\penalty0 (1):\penalty0 95--124,
  1999.

\bibitem[Hartline and Roughgarden(2009)]{HR09}
Jason~D. Hartline and Tim Roughgarden.
\newblock Simple versus optimal mechanisms.
\newblock In \emph{Proceedings 10th {ACM} Conference on Electronic Commerce},
  pages 225--234, 2009.

\bibitem[Mao et~al.(2018)Mao, Paes~Leme, and Schneider]{mao2018contextual}
Jieming Mao, Renato Paes~Leme, and Jon Schneider.
\newblock Contextual pricing for {L}ipschitz buyers.
\newblock In \emph{Advances in Neural Information Processing Systems}, pages
  5648--5656, 2018.

\bibitem[Medina and Mohri(2014)]{medina2014learning}
Andres~M Medina and Mehryar Mohri.
\newblock Learning theory and algorithms for revenue optimization in second
  price auctions with reserve.
\newblock In \emph{Proceedings of the 31st International Conference on Machine
  Learning (ICML-14)}, pages 262--270, 2014.

\bibitem[Medina and Vassilvitskii(2017)]{MedinaV17}
Andres~Mu{\~{n}}oz Medina and Sergei Vassilvitskii.
\newblock Revenue optimization with approximate bid predictions.
\newblock In \emph{Advances in Neural Information Processing Systems 30}, pages
  1856--1864, 2017.

\bibitem[Morgenstern and Roughgarden(2016)]{morgenstern2016learning}
Jamie Morgenstern and Tim Roughgarden.
\newblock Learning simple auctions.
\newblock In \emph{Proceedings of the Conference on Learning Theory (COLT)},
  pages 1298--1318, 2016.

\bibitem[Morgenstern and Roughgarden(2015)]{morgenstern2015pseudo}
Jamie~H Morgenstern and Tim Roughgarden.
\newblock On the pseudo-dimension of nearly optimal auctions.
\newblock In \emph{Advances in Neural Information Processing Systems}, pages
  136--144, 2015.

\bibitem[Myerson(1981)]{myerson1981optimal}
Roger~B Myerson.
\newblock Optimal auction design.
\newblock \emph{Mathematics of Operations Research}, 6\penalty0 (1):\penalty0
  58--73, 1981.

\bibitem[Paes~Leme et~al.(2016)Paes~Leme, P{\'a}l, and
  Vassilvitskii]{paes2016field}
Renato Paes~Leme, Martin P{\'a}l, and Sergei Vassilvitskii.
\newblock A field guide to personalized reserve prices.
\newblock In \emph{Proceedings of the 25th International Conference on World
  Wide Web (WWW)}, pages 1093--1102, 2016.

\bibitem[Roughgarden and Wang(2016)]{roughgarden2016minimizing}
Tim Roughgarden and Joshua~R. Wang.
\newblock Minimizing regret with multiple reserves.
\newblock In \emph{Proceedings of the 2016 ACM Conference on Economics and
  Computation (EC)}, pages 601--616. ACM, 2016.

\bibitem[Roughgarden et~al.(2012)Roughgarden, Talgam-Cohen, and
  Yan]{roughgarden2012supply}
Tim Roughgarden, Inbal Talgam-Cohen, and Qiqi Yan.
\newblock Supply-limiting mechanisms.
\newblock In \emph{Proceedings of the 13th ACM Conference on Electronic
  Commerce}, pages 844--861. ACM, 2012.

\end{thebibliography}
\bibliographystyle{plainnat}

\end{document}